\documentclass[letterpaper, 10 pt, conference, onecolumn]{ieeeconf}  

\IEEEoverridecommandlockouts                              

\usepackage{graphicx} 
\usepackage{epsfig} 
\usepackage{mathptmx} 
\usepackage{times} 
\usepackage{amsmath} 
\usepackage{amssymb}  
\usepackage{algpseudocode,algorithm2e}
\usepackage{subfig}

\title{\LARGE \bf
Event Driven CBBA with Reduced Communication }

\author{Vinita Sao, Tu Dac Ho, Sujoy Bhore and P.B. Sujit 
\thanks{This work is partially supported by UGC-JRF and partially supported by India-Russia DST-RSF grant. }
\thanks{Vinita Sao and P.B. Sujit are with the Department of Electrical Engineering and Computer Science, IISER Bhoapal, Bhopal -- 462066, India.}%
\thanks{Tu Dac Ho is with these affiliations: the Department of Information Security and Communication Technology (NTNU), Trondheim, Norway; and the Department of Electrical Engineering- The Arctic University of Norway (UiT).}
\thanks{Sujoy Bhore is with the Department of Computer Science, IIT Bombay, Mumbai, India.}%
}

\newtheorem{statement}{Statement}
\newtheorem{lemma}{Lemma}

\begin{document}

\maketitle
\thispagestyle{empty}
\pagestyle{empty}

\begin{abstract}
In various scenarios such as multi-drone surveillance and search-and-rescue operations, deploying multiple robots is essential to accomplish multiple tasks at once. Due to the limited communication range of these vehicles, a decentralised task allocation algorithm is crucial for effective task distribution among robots. The consensus-based bundle algorithm (CBBA) has been promising for multi-robot operation, offering theoretical guarantees. However, CBBA demands continuous communication, leading to potential congestion and packet loss that can hinder performance. In this study, we introduce an event-driven communication mechanism designed to address these communication challenges while maintaining the convergence and performance bounds of CBBA. We demonstrate theoretically that the solution quality matches that of CBBA and validate the approach with Monte-Carlo simulations across varying targets, agents, and bundles. Results indicate that the proposed algorithm (ED-CBBA) can reduce message transmissions by up to 52\%.
\end{abstract}

\section{Introduction}
\begin{figure}
    \centering
    \subfloat[]{%
        \includegraphics[scale=0.35]{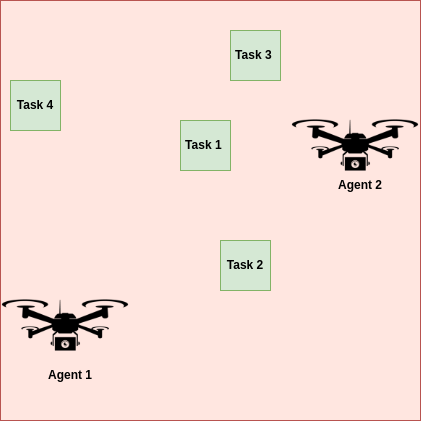}%
        \label{fig:task_agent}%
        }\hspace{0.5cm}
    \subfloat[]{%
        \includegraphics[scale=0.35]{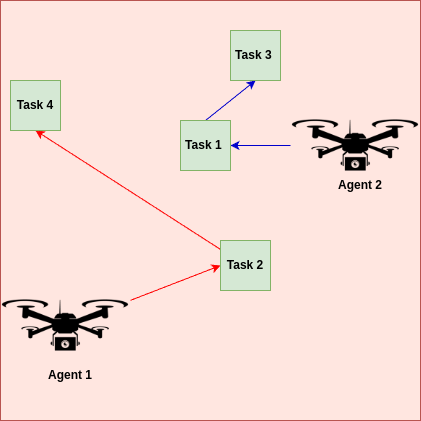} %
        \label{fig:allocation}%
        }%
    \caption{Task Allocation in Multi-UAV Systems: (a) Initial configuration with 4 tasks and 2 agents. (b) Allocation of tasks for the two vehicles.}\label{intro}
\end{figure}     
Recent technological advancements have rendered Unmanned Aerial Vehicles (UAVs) more economically viable, facilitating their widespread adoption across various fields. UAVs vary in size, with larger models used for rescue missions and smaller ones deployed in swarms for purposes like surveillance, data gathering \cite{dawson2021requirements}, agriculture \cite{basiri2022survey}, and defence applications such as reconnaissance and striking \cite{huang2016development, chen2022defense}. In systems involving multiple UAVs, coordinating diverse UAV types allows for leveraging their varied capabilities to efficiently manage tasks with different requirements. Nonetheless, a significant challenge in these systems is the optimal distribution of tasks among the UAVs to maximise their overall mission effectiveness. This problem, known as Multi-Robot Task Allocation (MRTA), is NP-Hard and remains complex \cite{gerkey2004formal, mrta_review}, even with a limited number of UAVs. Additionally, MRTA's complexity is compounded by factors such as fuel limitations, computing resources, network bandwidth, and evolving task demands.

In the scenario depicted in Figure \ref{fig:task_agent}, multiple tasks are present in the environment alongside UAVs (Unmanned Aerial Vehicles). The goal is to assign a sequence of tasks to each UAV in a manner that maximises the overall global score. As shown in Figure \ref{fig:allocation}, each UAV is assigned two tasks. A distributed allocation approach is preferred because it not only improves the global score but also balances the workload among the UAVs, preventing any single agent from being overloaded.

To address the Multi-Robot Task Allocation (MRTA) problem, various centralised approaches have been considered. One such approach is Mixed Integer Linear Programming (MILP)~\cite{jung2025milp, veeraswamy2025optimal}, which can yield optimal solutions. However, it does not scale well as the number of agents or tasks increases. Meta-heuristic techniques like Ant Colony Optimisation (ACO), Particle Swarm Optimisation (PSO), and Genetic Algorithms (GA) can provide near-optimal solutions while requiring less computational power compared to MILP approaches \cite{gao2021multi, zhou2024cooperative, wu2021multi, 9483937}. Despite their lower complexity, these techniques are typically centralised and encounter similar scalability and flexibility issues as the MILP methods. Consequently, there is an increasing demand for decentralised methods that can solve the MRTA problem more efficiently and robustly.
 
In contrast, market-based approaches have gained popularity for their adaptability in dynamic environments and their ability to provide efficient task allocation strategies \cite{gerkey2004formal, review_market}. These approaches include several variants, such as the Contract-Net Protocol \cite{1308816, smith1980contract, tiejun2007study, cervera2008agent}, simple auctions~\cite{gerkey2002sold, choi2009consensus, hoeing2007auction}, combinatorial auctions \cite{segui2014combinatorial}, and negotiation mechanisms \cite{sujit2006multiple, viguria2007set}, among others. Market-based methods have been extensively explored over the past two decades. Due to space limitations, we will only review the most relevant works here. For a more comprehensive overview of task allocation in market-based approaches, interested readers are referred to the detailed review articles and the citations therein \cite{gerkey2004formal, review_market}. Consensus-based approaches are also widely used as decentralised solutions in multi-UAV applications~\cite{9476858,review_market}. However, achieving convergence on a consistent situational awareness (SA) requires significant time and data. In this context, the Consensus-Based Bundle Algorithm (CBBA)~\cite{choi2009consensus} algorithm leverages both auction and consensus phases to reduce the computational load and ensure convergence in a connected network topology, even when agents have inconsistent SA. Unlike consensus-based methods that require consistent SA, CBBA allows agents to communicate their plans along with their local SA. Its auction mechanism naturally leads to a conflict-free allocation under these constraints. This two-phase approach ensures conflict-free task assignments with theoretical guarantees. 

Several variants of CBBA have been developed over the years to address different constraints, including coupled constraints \cite{ye2021decentralized}, asynchronous communication \cite{acbba}, and heterogeneous tasks and agents \cite{cbba_dynamic_communication, cooperation-constraints}. Building on these variants of the CBBA approach, several issues related to the MRTA problem have been addressed. However, a critical challenge remains relatively unexplored, namely, the impact of limited communication bandwidth and data packet collisions, which can cause extra delays in reaching a consensus, thereby hindering the timely arrival of a conflict-free assignment. In addition, the more frequently data is transmitted by the agents, the higher the possibility of failure and errors. In this paper, we propose an improved variant of the celebrated CBBA algorithm by developing an event-based communication strategy to reduce data communication. We refer to this variant as \textbf{ED-CBBA}, which stands for event-driven CBBA. Raja et al.~\cite{ca-cbba} introduced a communication-efficient strategy based on a deep reinforcement learning framework known as CA-CBBA. In this approach, an agent only broadcasts if it has a higher priority than its neighbours. However, this framework requires significant computational resources to learn the priority for communication reduction. In contrast, the event-driven communication strategy transmits messages only when the agent's information changes, thereby reducing unnecessary communication without requiring additional computational resources. We theoretically demonstrate that ED-CBBA achieves the same conflict-free task allocation as the original CBBA algorithm, while offering improved communication efficiency.

\section{Problem Formulation}
We consider a set of \(N_a\) autonomous agents  \(\mathcal{A}= \{1,\dots, N_a\}\) and set of \(N_t \) tasks \(\mathcal{T}= \{1,\dots, N_t\}\). We aim to assign tasks to agents such that each task is assigned to at most one agent, and each agent is assigned no more than \(K\) tasks. A set of \(K\) tasks can be defined as the bundle \(b_i\) for the agent \(i \in \mathcal{A}\). The objective is to maximise the total score given as: 
\begin{align}
\text{Objective: } \max & \sum_{i=1}^{N_a} \left(\sum_{j=1}^{N_t} c_{ij}( p_i) x_{ij}\right) \label{eq:objective_function}\\ 
\textbf{Subject to:} & \\ &\sum_{i=1}^{N_a} x_{ij} \leq 1, \quad \forall j \in \mathcal{T} \label{eq:constraint1}\\ &\sum_{j=1}^{N_t} x_{ij} \leq K, \quad \forall i \in \mathcal{A} \label{eq:constraint2}\\ &\sum_{i=1}^{N_a} \sum_{j=1}^{N_t} x_{ij} = N_{min}, \quad \label{eq:constraint3}
\end{align}
where  $N_{min} \triangleq \min \{N_t, K N_a\}$, \(x_{ij} \in \{0, 1\} \quad \forall (i,j) \in \mathcal{A} \times \mathcal{T}\). In equation \eqref{eq:objective_function}, \(c_{ij}\) represents the score for task \(j\) assigned to agent \(i\), while \(p_i\) denotes the ordered list of tasks which we will refer to as a path. The decision variable \(x_{ij}\) indicates whether agent \(i\) has been assigned task \(j\) (1 if assigned, otherwise 0).
Equation \eqref{eq:constraint1} ensures conflict-free assignment by guaranteeing that each task is assigned to at most one agent. To limit each agent to a maximum of \(K\) tasks, equation \eqref{eq:constraint2} is applied. Finally, equation \eqref{eq:constraint3} specifies that a maximum number of \(N_{min}\) tasks can be assigned, which is determined by the minimum of the total number of tasks or the total number of tasks that can be accommodated based on the agents' capacities. 

CBBA is an iterative process consisting of two phases. In the first phase, known as the auction phase, agents select tasks based on a bidding process. In the second phase, called the consensus phase, agents share their knowledge about tasks with their neighbours. The process iterates between these two phases until convergence is achieved; convergence refers to a conflict-free task allocation among all agents.

\subsection{Auction Phase}

In this phase, agents use a scoring function that follows the principle of Diminishing Marginal Gain (DMG) to calculate bids on tasks. According to DMG, as agents continue adding tasks to their list, the marginal gain of each additional task decreases monotonically. Because the value of each new task decreases, it encourages the tasks to be distributed among agents rather than one agent taking many.

Using the scoring function, agents bid on tasks and manage several lists. Each agent \( i \in A \) keeps a \textit{bundle} \( b_i \), which is an unordered set of tasks. Since the goal is to determine the sequence of tasks for each agent, a \textit{path} \( p_i \), represented as an ordered list, is recorded. A list \( y_i \) of length \( N_t \) is used to store the bids on each task. A separate list \( z_i \) is used to record the winner for each task. Additionally, a list \(s_i\) is maintained to record timestamps, allowing agents to track when information is received, ensuring that they are up-to-date.

\begin{equation}
\hspace{-6em}c_{ij}(b_i) = 
\begin{cases}
    \text{0, } &  \hspace{-8em}
    \text{if } j \in b_i \\
    
    \max_{n \leq |p_i| + 1} \left( S_i^{p_i \oplus_n j} - S_i^{p_i} \right ) , Otherwise
    
\end{cases}\label{score}
\end{equation}

In equation \eqref{score}, \(S_i^{p_i}\) gives the score for each task in the bundle \(b_i\) for agent \(i\) following the path \(p_i\) and the operator \(\oplus_n\) denotes that the task is added at the \(n\)-th position in \(p_i\) list. If the task is already in the bundle, the score will be 0; otherwise, the agent will calculate the best index \(n\), which is the position where adding task \(j\) results in the highest total score of the bundle. After bundle construction, agents perform consensus.

\subsection{Consensus Phase}\label{consensus}

\begin{figure}
    \centering
    \subfloat[]{%
        \includegraphics[scale=0.26]{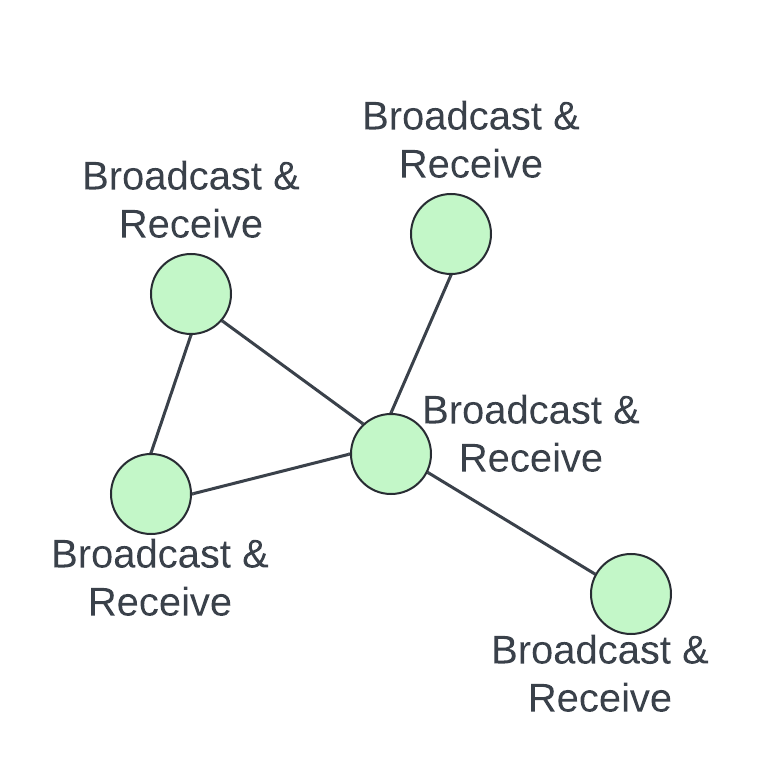}%
        \label{fig:cbba}%
        }%
    \subfloat[]{%
        \includegraphics[scale=0.25]{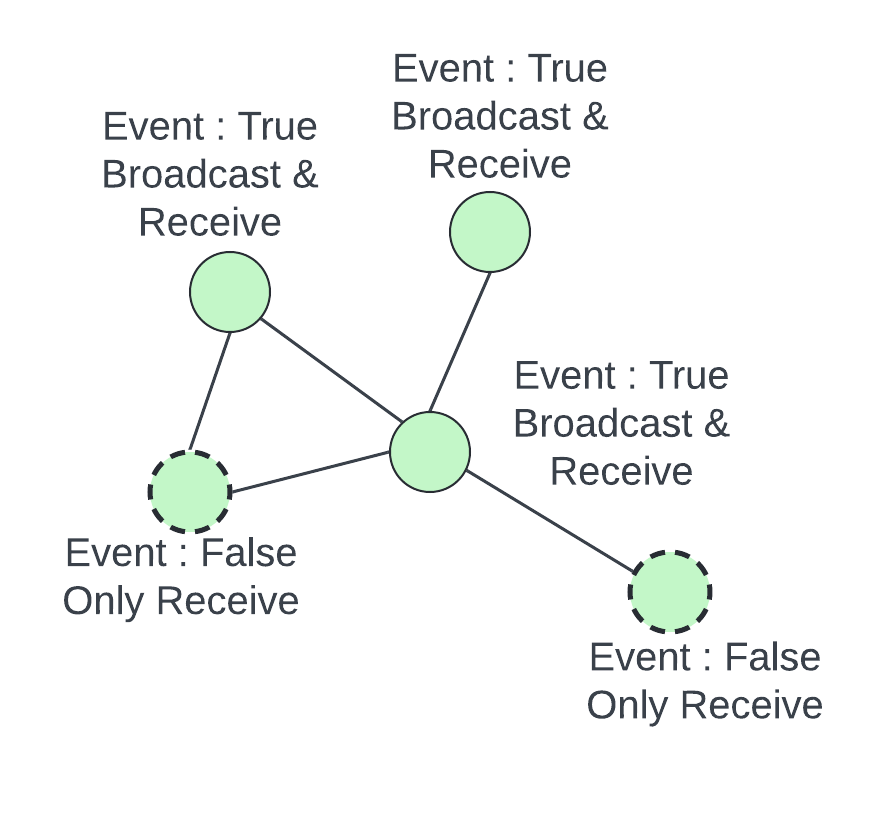} %
        \label{fig:eventdriven}%
        }%
    \caption{Schematic depicting communication mechanism in (a) CBBA, (b) ED-CBBA.}
\end{figure}

This is the second phase of the iteration, where agents share and receive information with their neighbours to get a conflict-free assignment. We can define the agents as the nodes of the graph \(G(\tau)\) to check their local network. To check the neighbourhood at real-time, \(g_{ik}(\tau)\) is used; if  \(g_{ik}(\tau) = 1,\) then agents are neighbours at time \(\tau\), otherwise not. An agent can have multiple neighbours depending on the communication topology. When all neighbours broadcast simultaneously, network congestion increases, leading to packet loss and message corruption \cite{ca-cbba}. Additionally, broadcasting and receiving messages in each iteration along with the associated processing consumes significant energy, reducing the flying time of the vehicle. The loss of information during communication results in delays, causing agents to take more time to reach a consensus.

To address this issue, we introduce event-driven communication, which we call ED-CBBA. In ED-CBBA, agents will broadcast only when the event occurs; otherwise, they will receive the information. In Figure \ref{fig:cbba} and \ref{fig:eventdriven}, the nodes represent agents, and two agents are considered neighbours if there is a direct edge between them. In Figure \ref{fig:cbba}, it is shown that in CBBA, all agents send messages, resulting in a network that is overloaded with messages, which ultimately affects the consensus phase. In Figure \ref{fig:eventdriven} for ED-CBBA, only agents for which the event condition is true send messages, while others only receive. This approach reduces the number of message packets exchanged in each iteration without compromising the information of neighbouring agents. Consequently, this minimises redundant message processing, saving both battery and time.
Algorithm 1 shows the consensus process for the agent \(i \in \mathcal{A}\) at some iteration \(t\). Before the broadcast, at line 1, the agent will check for the occurrence of the event as follows : 
\begin{equation}
    \text{event(t)} =
    \begin{cases}
    \text{True,} & \text{if } y_i(t) \neq y_i(t-1)\\
                & \phantom{\text{if }} \text{or } z_i(t) \neq z_i(t-1)\\
    \text{False,} & \text{otherwise}.
    \end{cases}   
\end{equation}

If  \(event(t) = True\), it means an agent has some updated information to share. Here, updated information refers to any change in the bid list or the winner list. If \( event(t) = \text{False} \), the agent will only receive information from its neighbours. Since the agent is not sharing information because the information it has at iteration \( t \) has already been broadcast at iteration \( t-1 \) or in previous iterations, transmitting it again will not affect the neighbours' knowledge. At line 2, the agent identifies its neighbours at time \(\tau\); based on the received information from the neighbours, the agent performs some operations for each task \(j \in \mathcal{T}\):

\subsubsection{\textbf{Update}} If the winner of a task differs between the agent and the sender, the agent checks whether another agent has overbid the task. If so, and if the agent itself has the task, it will discard the task from its bundle and update its information accordingly. Since the tasks in an agent's bundle depend on their assigned paths, removing a task invalidates the scores of subsequent tasks. As a result, all tasks following the discarded one must also be removed from the bundle.

\subsubsection{\textbf{Reset}} If no clear information about a task is available, the agent resets its corresponding values to null. Additionally, if the task is present in the agent's bundle, the agent discards it along with all subsequent tasks in the bundle.

\subsubsection{\textbf{Leave}} When an agent has the same or more recent knowledge than the sender, it simply leaves the information about that task unchanged.

After the completion of this phase, agents will have information about their neighbours. Accordingly, they will update their bundle, bid, and winner lists. These two phases will continue to iterate until a conflict-free allocation is achieved.

\begin{algorithm}  
\hrulefill
\caption{Conflict Resolution Phase: for Agent \(i\) at Iteration \(t\).}

\SetAlgoLined 
\KwIn{ \( y_k(t), z_k(t), s_k(t) \quad \text{for} \quad g_{ik}(\tau) = 1 \quad \forall k \in \mathcal{A} \) }
\KwOut{Agent \(i\) update lists based on received information.\;}

1. \If{\(\text{event(t)} = \text{true}\)}{
    Agent \(i\) broadcasts its information lists\;
}
2. \For{\(k \in \mathcal{A}\) where \(g_{ik}(\tau) = 1\)}{
    Receive information from neighbor \(k\)\;
    
    3. \For{\(j \in \mathcal{T}\)}{
        Based on the received data for each task \(t\), perform the following operations:
        
        \textbf{Update:} \(y_{ij} = y_{ik}, z_{ij} = z_{ik} \)\;
        
        \textbf{Reset:} \(y_{ij} = 0, z_{ij} = \phi\)\;
        
        \textbf{Leave:} \(y_{ij} = y_{ij}, z_{ij} = z_{ij}\)\;
    }
}
\end{algorithm} 
\hrulefill

 \section{Convergence Analysis}
A crucial aspect of convergence in MRTA is the property of Diminishing Marginal Gain (DMG). To demonstrate convergence, \cite{choi2009consensus} assumes that the network is static and connected, and the consensus process is synchronised.

When the event-driven approach is introduced, the convergence property is maintained. Since the scoring scheme satisfies the DMG property (as proven in \cite{choi2009consensus}), it holds when the network remains static and connected. The consensus process continues to be synchronised; ED-CBBA preserves the convergence. To see that \cite{choi2009consensus} has proved some statements after the completion of some iteration \(t\) for the agent \(i^*\) and the task \(j^*\) bid value \(y_{i^*j^*}(t) = c_{i^*j^*}\) and the winner for the task \(z_{i^*}j^* = i^*\) than we can say under DMG scoring scheme:
\begin{statement}
 The bid agent \(i^*\) places on the task \(j^*\) will be the highest bid over all the bids on the task at iteration \(t\): \[
     c_{i^*j^*} \geq y_{ij}(t) \quad \forall (i,j) \in \mathcal{A} \times \mathcal{T} \]
    \end{statement}
    \begin{statement}
 Entries do not change over time: \[
    z_{i^*}j^*(s) = z_{i^*}j^*(t) \quad y_{i^*}j^*(s) = y_{i^*}j^*(t), \quad \forall s \geq t
    \] 
      \end{statement}
      
     \begin{statement}
  The bid agent \(i^*\) places on the task \(j^*\) will remain the same throughout, and this can not be overbid by the other agents:\[
    y_{i^*}j^*(s) = y_{i^*}j^*(t) >= y_{i}j^*(t) \quad \forall s \geq t, \quad \forall i \in \mathcal{A}
    \] 
      \end{statement}
 
     \begin{statement}
After \(D\) iterations; every agent will have agreed on the assignment \(i^*, j^*\): \[
 z_{i^*}j^*(t+D) = i^* \quad  y_{i^*}j^*(t+D) = y_{i^*}j^*(t) \]
      \end{statement}
      
All four statements hold when the scoring scheme satisfies the DMG property. To prove that ED-CBBA preserves convergence with a reduced number of communications, we will show that by modifying the conflict resolution phase, the information across the agents will not change, and hence, it will not affect the allocation.

\begin{lemma}
    The Event-Driven Consensus-Based Bundle Algorithm (ED-CBBA) converges with an upper bound of \(N_{\min}D\) when the scoring scheme follows the DMG principle, the communication network is static and connected with diameter \(D\), and the conflict resolution phase is synchronous:
    \[
    T_c \leq N_{\min}D
    \]
\end{lemma}

\begin{proof}
    By proving statement 4, which is, on one agent task pair \( (i^*, j^*) \), all agents will agree after \(D\) iterations. We can generalise this for all the agent task pairs while ensuring all the agents reach a consensus on their respective assignments. 
   
    Since we have assumed that consensus is synchronous in static environment, we can refer \(\tau\) as iteration \(t\). From statements 2 and 3, \( (i^*, j^*) \) is the agent task pair with the highest bid, which will be consistent throughout the later iterations. After iteration \(t\), agent \(i^*\) has no updated information, so it will not broadcast. As we assumed that only direct neighbours/one-hop neighbours could receive the information, in that \(t\) iteration, all the t-hop neighbours of the agent \(i^*\) will have this information; therefore, even though \(i^*\) stops broadcasting, all agents with this information will relay it to their immediate neighbours and so on. 

    Assuming \( (i^*, j^*) \) is the only update available to the agents, they will broadcast this information at least once. This will allow the information to pass through \(1\)-hop neighbours to \(2\)-hop neighbours and eventually to \(D\)-hop neighbours. Consequently, after \(D\) iterations, all agents will agree on the assignment \( (i^*, j^*) \). Hence, statement 4 holds for the Event-Driven Consensus-Based Bundle Algorithm (ED-CBBA). Thus, after \(D\) iterations, agreement on \( (i^*, j^*) \) is reached. According to statement 3, these entries will remain unchanged in subsequent iterations. This principle also applies to the \(m \in \mathcal{T}\) tasks, indicating that \(m\) tasks will be assigned conflict-free in \(mD\) iterations. Therefore, \(N_{min}\) tasks will be assigned in \textbf{\(N_{min}D\)} iterations.
\end{proof}

\subsection{Performance Bound}
The performance bound of CBBA is shown to guarantee 50\% optimality, assuming that agents possess accurate knowledge of SA and utilise a DMG scoring function. As we have considered a scoring function that is the same as in  \cite{choi2009consensus}, the modification in the communication mechanism preserves the performance properties of \cite{choi2009consensus}.


\section{Results and Discussion}

\begin{figure}
    \centering
    \includegraphics[height=5.2cm]{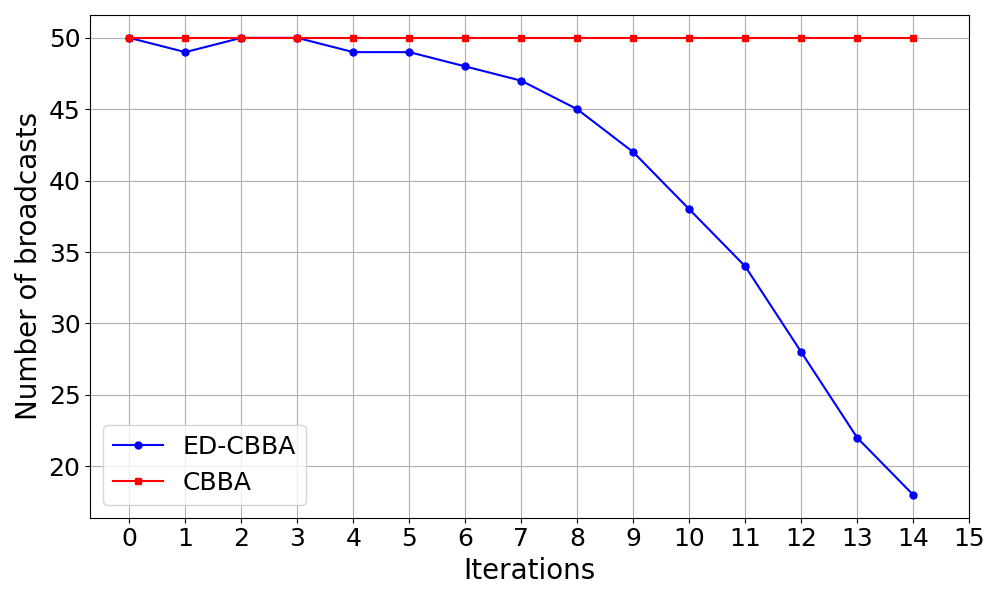}
    \caption{Number of broadcasts in each iteration: The red line represents CBBA, while the blue line represents ED-CBBA.}
    \label{fig:broadcast_each_iter}
\end{figure}


\begin{figure*}
    \centering
    \subfloat[]{%
        \includegraphics[width=0.2\textwidth]{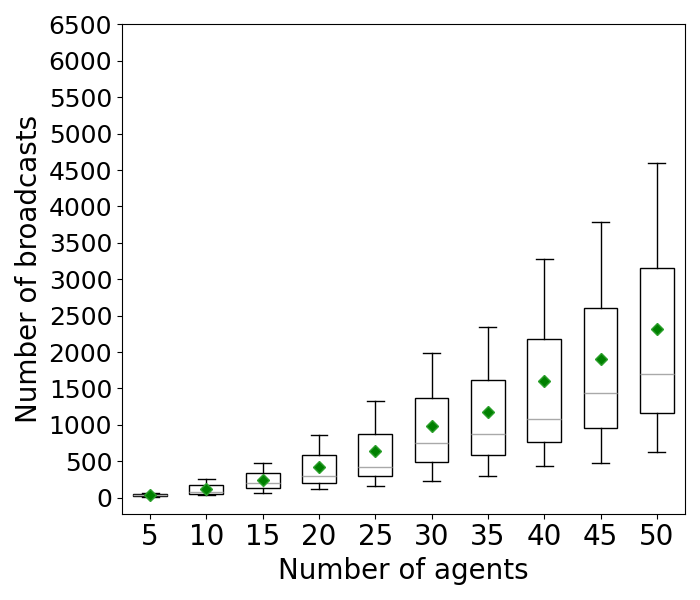}%
        }%
        \subfloat[]{%
        \includegraphics[width=0.2\textwidth]{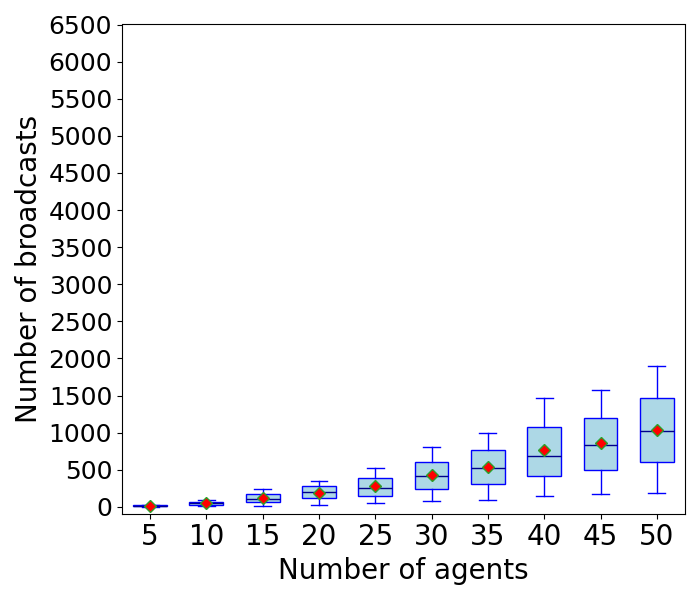}%
        }%
    \subfloat[]{%
        \includegraphics[height=3.0cm]{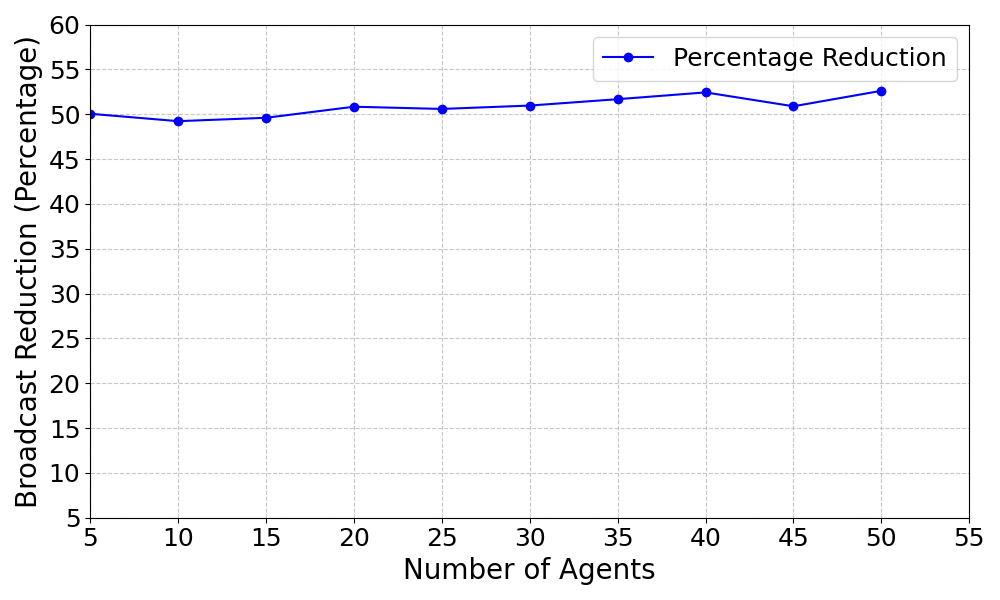}%
        }%
    \subfloat[]{%
        \includegraphics[height=3.0cm]{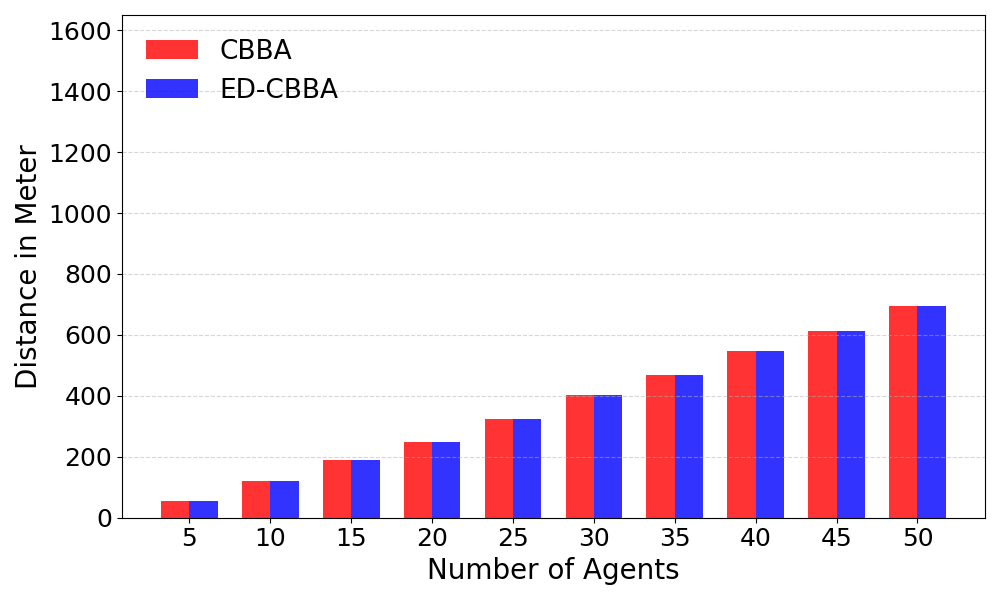}%
        }%
   \caption{Comparison of CBBA and ED-CBBA in terms of the number of broadcasts and task assignments for \(K = 1\) :
    (a) Box plot showing the number of broadcasts in CBBA,
    (b) Box plot showing the number of broadcasts in ED-CBBA,
    (c) Percentage reduction in broadcasts for ED-CBBA compared to CBBA, 
    (d) Total distance travelled by agents while completing tasks.}
    \label{fig:ed1}
\end{figure*}

\begin{figure*}
    \centering
    \subfloat[]{%
        \includegraphics[width=0.2\textwidth]{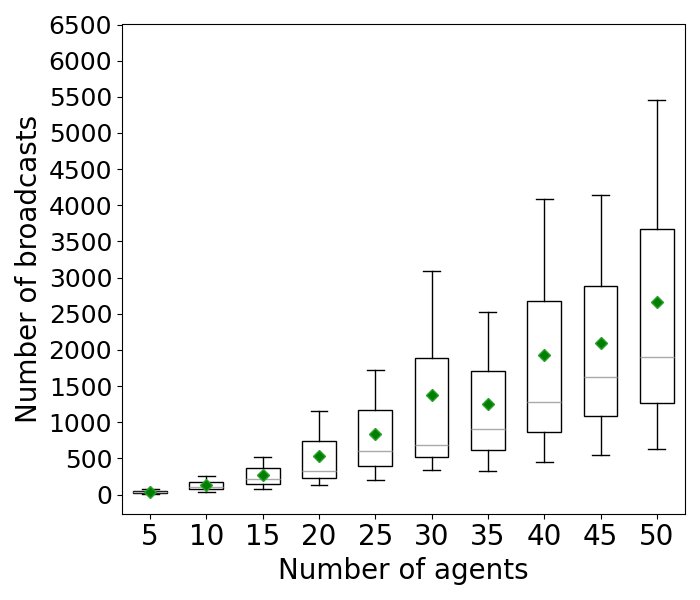}%
        }%
        \subfloat[]{%
        \includegraphics[width=0.2\textwidth]{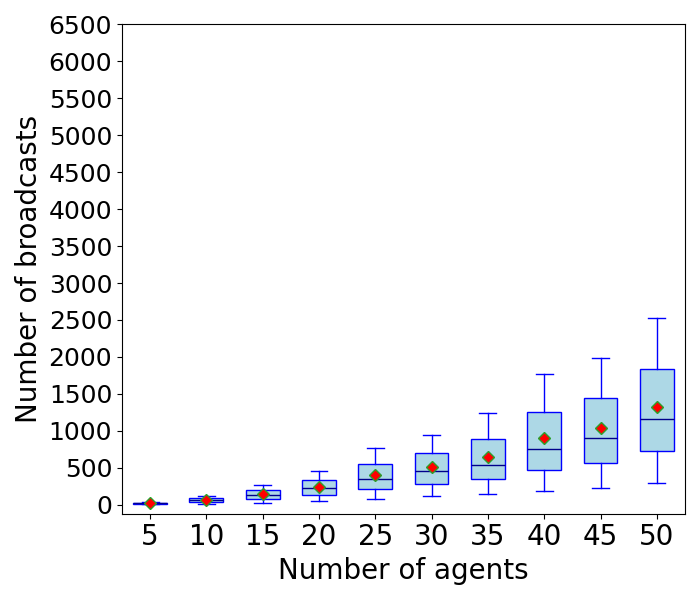}%
        }%
    \subfloat[]{%
        \includegraphics[height=3.0cm]{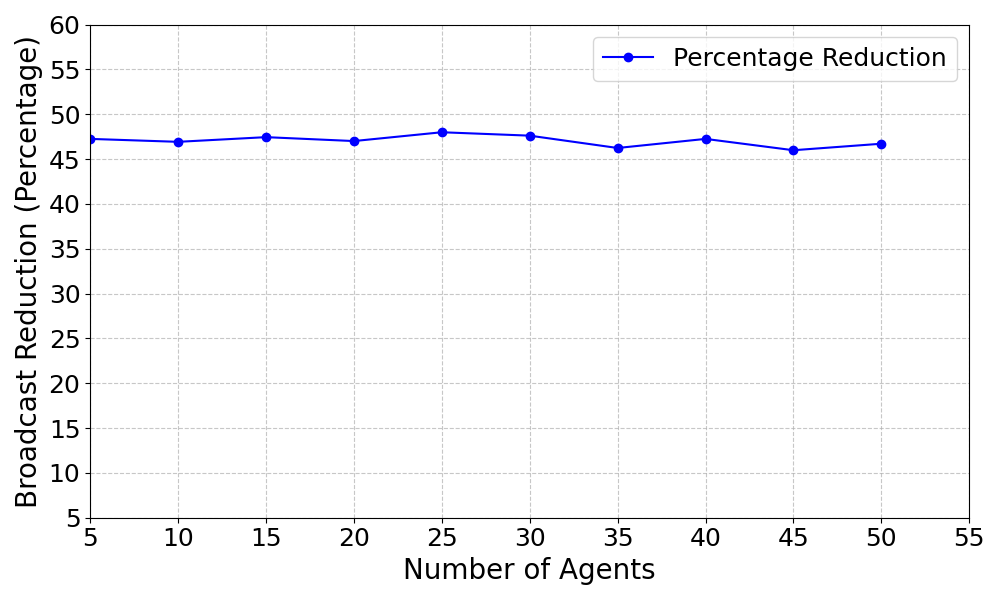}%
        }%
    \subfloat[]{%
        \includegraphics[height=3.0cm]{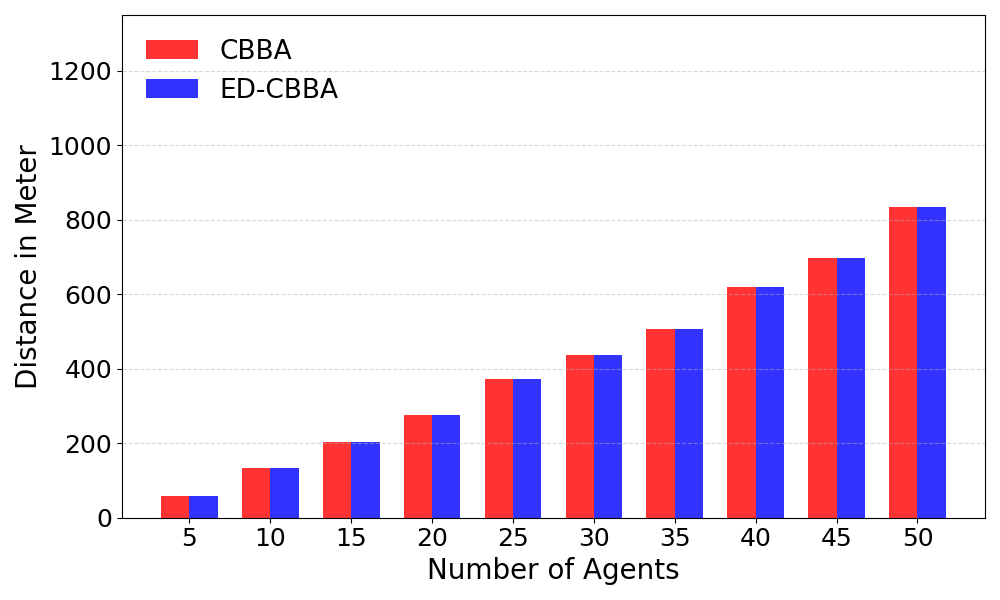}%
        }%
   \caption{Comparison of CBBA and ED-CBBA in terms of the number of broadcasts and task assignments for \(K = 2\) :
    (a) Box plot showing the number of broadcasts in CBBA,
    (b) Box plot showing the number of broadcasts in ED-CBBA,
    (c) Percentage reduction in broadcasts for ED-CBBA compared to CBBA, 
    (d) Total distance travelled by agents while completing tasks.}

    \label{fig:ed2}
\end{figure*}

\begin{figure*}
    \centering
    \subfloat[]{%
        \includegraphics[width=0.2\textwidth]{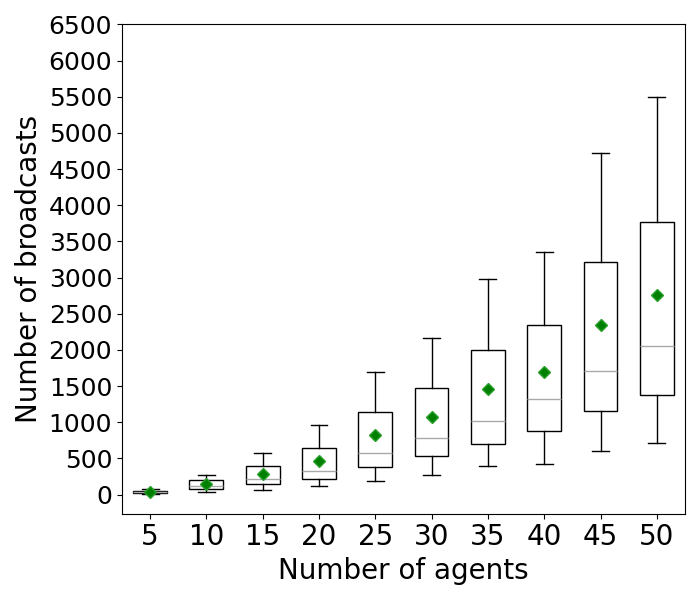}%
        }%
        \subfloat[]{%
        \includegraphics[width=0.2\textwidth]{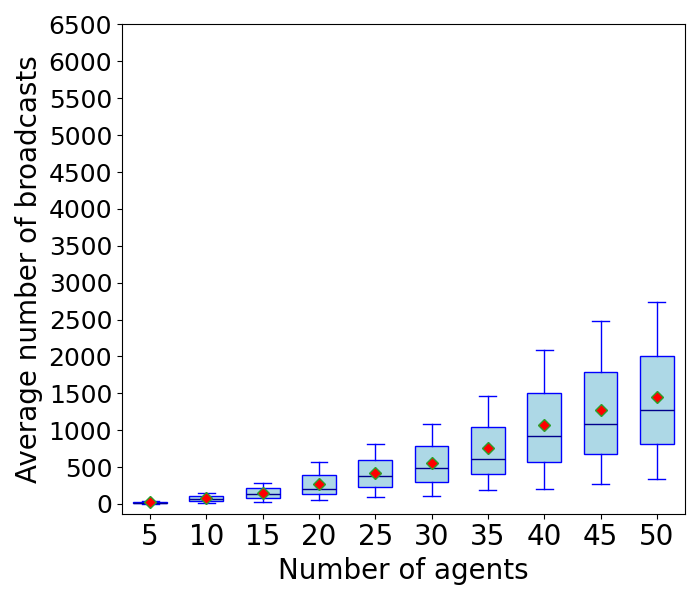}%
        }%
    \subfloat[]{%
        \includegraphics[height=3.0cm]{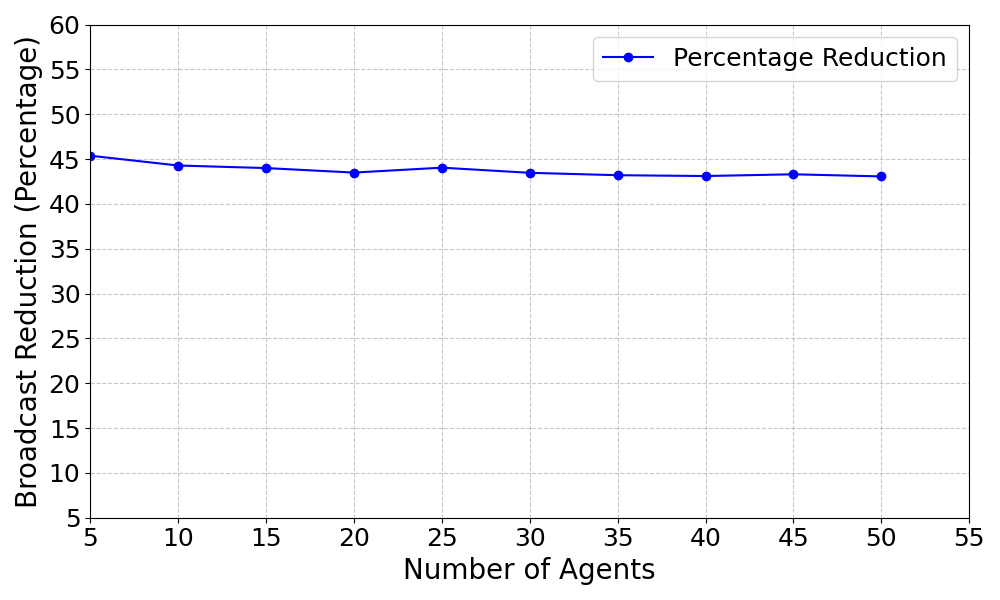}%
        }%
    \subfloat[]{%
        \includegraphics[height=3.0cm]{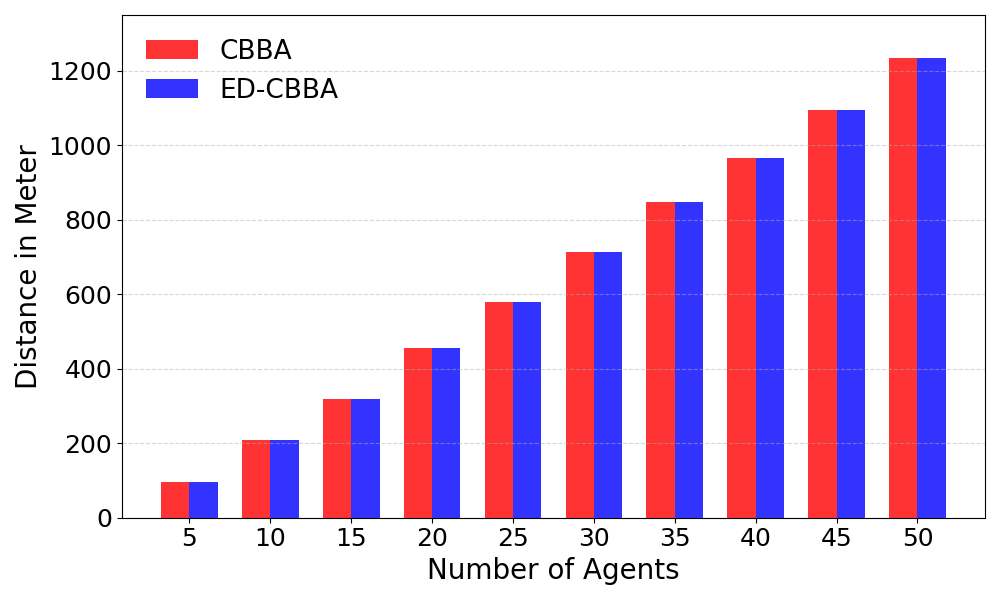}%
        }%
    \caption{Comparison of CBBA and ED-CBBA in terms of the number of broadcasts and task assignments for \(K = 3\) :
    (a) Box plot showing the number of broadcasts in CBBA,
    (b) Box plot showing the number of broadcasts in ED-CBBA,
    (c) Percentage reduction in broadcasts for ED-CBBA compared to CBBA, 
    (d) Total distance travelled by agents while completing tasks.}

    \label{fig:ed3}
\end{figure*}
We evaluate the performance of ED-CBBA through Monte Carlo simulations conducted in a \(100\,m \times 100\,m\) 2-D space, where both agents and tasks are generated. Neighbours are determined based on a 10-meter distance threshold, ensuring that every agent is within range of at least one other agent, thus forming a connected graph. To maintain this connectivity, agent locations are generated such that each agent is within 10 meters of at least one agent. Additionally, the base score for each task is set to 100. We varied the number of agents between 5 to 50 and tested with bundle sizes of \(K=1\) to \(K=8\), running 100 simulations for each configuration.

 In the standard CBBA model, every agent broadcasts a message during each iteration, resulting in \(N_a\) broadcasts for \(N_a\) agents, which contributes to a higher communication load. Figure \ref{fig:broadcast_each_iter} illustrates the broadcasts in each iteration for both algorithms. The red line represents the broadcasts for CBBA, which remain constant across all iterations. In contrast, the blue line, representing ED-CBBA, shows a decrease in broadcasts as the iterations progress, indicating that the number of broadcasts is not necessarily the same in each iteration. As the agents approach convergence, the broadcasts in ED-CBBA decrease.

To demonstrate the significant reduction in broadcasts without compromising task allocation, Figures~\ref{fig:ed1}, \ref{fig:ed2}, and \ref{fig:ed3} each illustrate four aspects for varying numbers of agents (\(N_a = 5\) to \(N_a = 50\)): the distribution (minimum, maximum, median, and mean) of the number of broadcasts in CBBA, the same in ED-CBBA, the percentage reduction in broadcasts achieved by ED-CBBA, and the total distance travelled by agents as a measure of task allocation cost. Figure \ref{fig:ed1} illustrates the broadcast for bundle length 1; the box plots in Figures \ref{fig:ed1}(a) and \ref{fig:ed1}(b) provide insights into the distribution of broadcasts for CBBA and ED-CBBA, respectively. The diamond marker represents the mean, while the horizontal line inside each box indicates the median. For ED-CBBA in Figure~\ref{fig:ed1}(b), the proximity of the mean and median suggests consistent performance with low variability in the number of broadcasts across simulations. In contrast, for CBBA in Figure~\ref{fig:ed1}(a), the larger gap between the mean and median highlights greater variability in broadcasts, likely due to fluctuations in the number of iterations required for convergence. Although CBBA and ED-CBBA require the same number of iterations, ED-CBBA achieves fewer broadcasts overall because not all agents communicate in every iteration. As evident in Figures~\ref{fig:ed1}(a) and (b), the average number of broadcasts in ED-CBBA is consistently lower than CBBA for all agent configurations. Figure~\ref{fig:ed1}(c) highlights the percentage reduction in broadcasts achieved by ED-CBBA, reaching up to 52\%, without compromising the allocation quality. To validate this claim, Figure~\ref{fig:ed1}(d) presents the total distance covered by agents during task completion. The results show that the average distance travelled by agents remains the same for both CBBA and ED-CBBA, demonstrating that ED-CBBA reduces communication overhead without impacting task allocation performance.

In this paper, we used the average distance travelled by agents as the primary metric because it directly reflects task allocation efficiency and overall mission performance. Since ED-CBBA primarily focuses on reducing communication overhead without altering task execution, we aimed to demonstrate that the task allocation quality remains consistent with that of CBBA. We acknowledge that other flight parameters may deteriorate due to variations in coordination strategies. However, our focus was on communication efficiency and task allocation stability, and our results indicate that ED-CBBA does not introduce additional travel costs.

\begin{figure}
    \centering
    \includegraphics[width=0.6\linewidth]{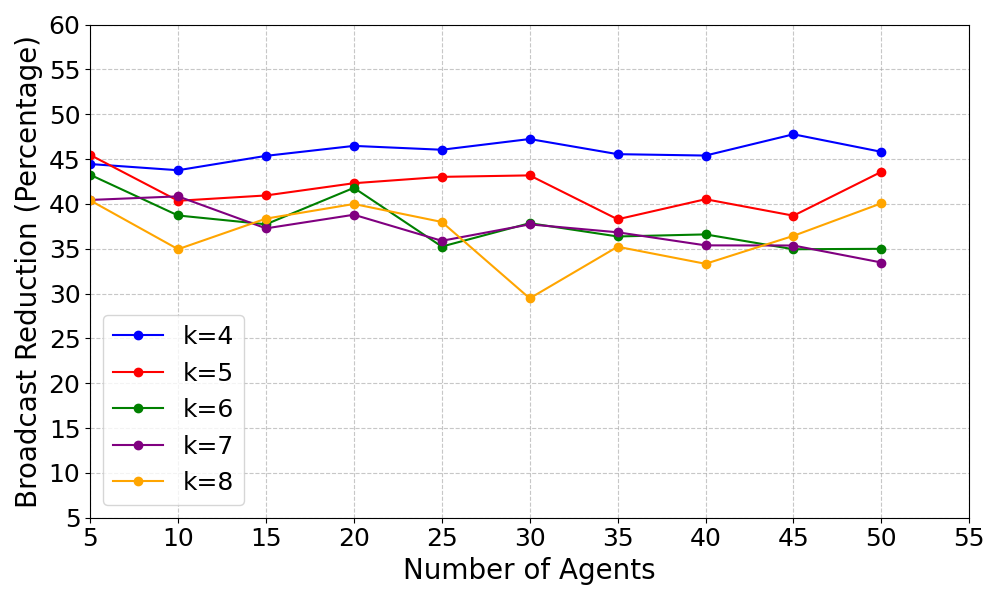}
    \caption{Reduction in broadcast in ED-CBBA with respect to CBBA in percentage for different bundle length, \(k = 4\) to \(k = 8\).}
    \label{fig:ed4}
\end{figure}
\begin{figure}[hb]
    \centering
    \includegraphics[width=0.6\linewidth]{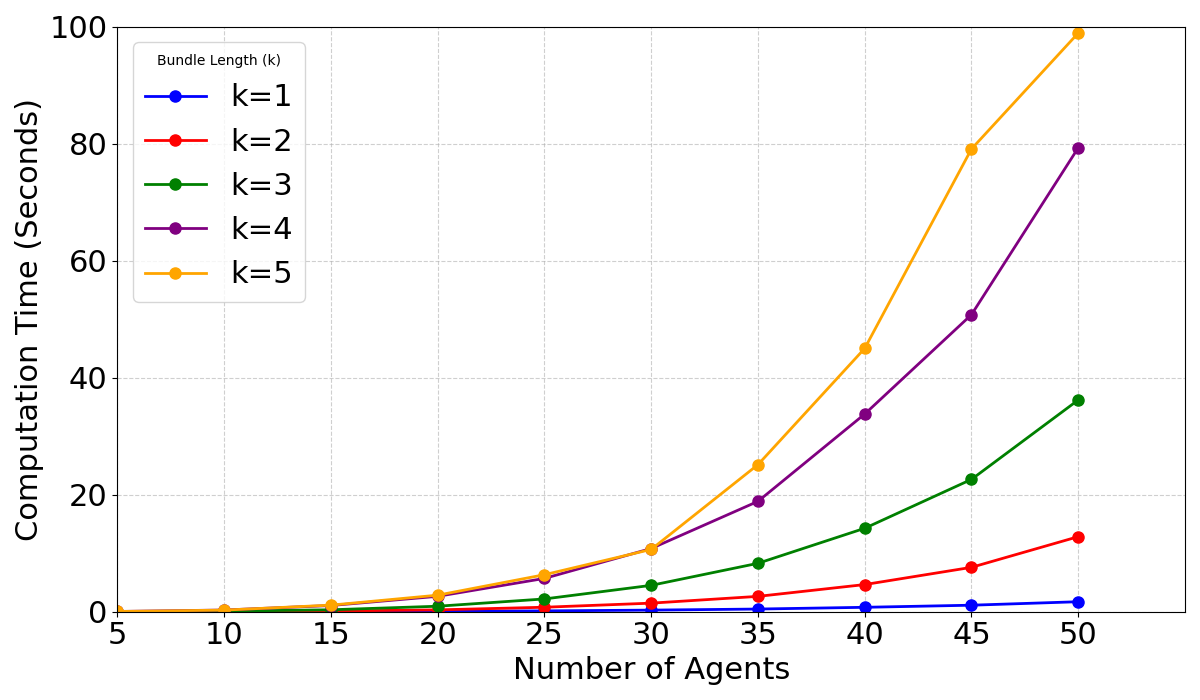}
    \caption{Computation time in seconds for the varying number of agents and bundle lengths \(k = 1\) to \(k = 5\).}
    \label{fig:cmp_time}
\end{figure}

We performed simulations for various values of bundle length \(K\). Figure \ref{fig:ed2} shows the results for a bundle length of 2, where we achieved a reduction of up to 48\%. For a bundle length of 3, as illustrated in Figure \ref{fig:ed3}, the reduction reached up to 45\%. Figure \ref{fig:ed4} displays the percentage reduction for bundle lengths ranging from \(K=4\) to \(K=8\), showing reductions of up to 40\%. These results highlight that ED-CBBA significantly reduces communication overhead, achieving up to \textbf{52\%} fewer broadcasts, across various agent configurations and bundle lengths, while maintaining the same task allocation performance as the standard CBBA.

\subsection{Computation Time}

All simulations and computations related to ED-CBBA were performed on a high-performance desktop system equipped with an AMD Ryzen 9 7950X (16 cores, 32 threads) and 128 GB RAM, ensuring that the results reflect the computational complexity of the approach. Since ED-CBBA performs offline task assignments, agents compute their bundles before execution. As shown in Figure~\ref{fig:cmp_time}, computation time increases with the bundle length and the number of agents. For smaller bundle lengths, computation remains below one second, while the highest observed time was 98.88 seconds for 50 agents with the maximum bundle length, which remains feasible for real-time execution.

We acknowledge that on-board computing devices used in UAVs often operate under more constrained computational and memory resources. While ED-CBBA has not been explicitly tested on Jetson hardware, performance differences are expected due to processing power and memory bandwidth variations. A detailed real-time evaluation of embedded platforms is an important direction for future work.

\section{Conclusion}
In this paper, we propose event-driven communication to reduce the number of communication messages to achieve consensus in the CBBA algorithm. The event-driven mechanism is shown to have reduced the number of messages to be broadcast significantly without additional computational complexity. The results demonstrate that the improved CBBA algorithm achieves similar task allocation performance to the original CBBA but with a 52\% reduction in communication requirements. We evaluated the algorithm through simulations, showing its effectiveness in reducing communication overhead.

The proposed algorithm can be further extended to an asynchronous communication framework, taking into account real-world network errors such as bit errors and packet loss. The algorithm can be integrated with NS-3 simulators to further study the effect of communication issues associated with multi-robot systems using CBBA.

\bibliographystyle{IEEEtran} 
\bibliography{references}

\end{document}